%% file: neurips_2023.tex
\documentclass{article}

\PassOptionsToPackage{numbers, sort&compress, comma, square}{natbib}
\usepackage[preprint]{neurips_2023}

\usepackage[utf8]{inputenc} 
\usepackage[T1]{fontenc}    
\usepackage{hyperref}       
\usepackage{url}            
\usepackage{booktabs}       
\usepackage{amsfonts}       
\usepackage{nicefrac}       
\usepackage{microtype}      
\usepackage{xcolor}         

\usepackage{xspace}
\usepackage{tabularx}
\usepackage{wrapfig,lipsum}
\usepackage{graphicx} 
\usepackage{multirow}
\usepackage{CJK}
\usepackage{colortbl}
\usepackage{caption}
\usepackage{subcaption}
\usepackage{enumitem}
\usepackage{latexsym}
\usepackage{amsmath}
\usepackage{arydshln}
\usepackage{amsthm}
\usepackage{mathtools}

\newtheorem{theorem}{Theorem}[section]
\newtheorem{proposition}[theorem]{Proposition}

\usepackage[utf8]{inputenc} 
\usepackage[T1]{fontenc}    
\usepackage{hyperref}       
\usepackage{url}            
\usepackage{booktabs}       
\usepackage{amsfonts}       
\usepackage{nicefrac}       
\usepackage{microtype}      
\usepackage{xcolor}         

\usepackage{lettrine}
\usepackage{algorithm} 
\usepackage[noend]{algpseudocode}
\usepackage[normalem]{ulem}
\usepackage{xspace}

\newcommand{\method}{\textsc{CoDAT}\xspace}
\definecolor{myOran}{RGB}{255,163,76}
\definecolor{myBlue}{RGB}{9, 163, 241}

\title{Optimizing Non-Autoregressive Transformers with Contrastive Learning}

\author{Chenxin An$^{1,2}$, Jiangtao Feng, Fei Huang$^{3}$, Xipeng Qiu$^{2}$, Lingpeng Kong$^{1}$ \\
  $^1$The University of Hong Kong \\
  $^2$Fudan University \\
  $^3$The CoAI group, Tsinghua University \\
  \texttt{\{cxan20, xpqiu\}@fudan.edu.cn}
  \texttt{jiangtaofeng0906@gmail.com} \\ \texttt{ huangfei382@163.com}, \texttt{lpk@cs.hku.hk}\\
 }
 
\begin{document}

\maketitle
\begin{abstract}
Non-autoregressive Transformers (NATs) reduce the inference latency of Autoregressive Transformers (ATs) by predicting words all at once rather than in sequential order. They have achieved remarkable progress in machine translation as well as many other applications.
However, a long-standing challenge for NATs is the learning of multi-modality data distribution, which is the main cause of the performance gap between NATs and ATs.
In this paper, we propose to ease the difficulty of modality learning via sampling from the model distribution instead of the data distribution.
We derive contrastive constraints to stabilize the training process and integrate this resulting objective with the state-of-the-art NAT architecture DA-Transformer~\citep{huang2022directed}.
Our model \method is examined on 3 different tasks, including machine translation, text summarization, and paraphrasing with 5 benchmarks. Results show that our approach outperforms previous non-autoregressive baselines by a significant margin and establishes new state-of-the-art results for non-autoregressive transformers on all the benchmarks.
\end{abstract}

\input{01_intro}
\input{02_bg}

\input{03_method}

\input{04_exps}
\input{05_relatedwork}

\section{Conclusion}
In this paper, we develop the first contrastive learning framework for  non-autoregressive generation. The goal of our contrastive learning objective is alleviating the modality learning of NATs by using training samples from the model distribution instead of complex data distribution. In order to stabilize training, we derive contrastive constraints which finally lead to a contrastive learning loss.  Experiments on three text generation tasks demonstrate that our method achieves the best results on all these tasks for non-autoregressive models. Our work has many \textbf{limitations} please refer to Appendix~\ref{limit} for details.

\bibliography{neurips_2023}

\newpage
\appendix

\input{06_appendix}

\end{document}

%% file: 01_intro.tex
\section{Introduction}
Autoregressive Transformers (ATs) have become the dominant architecture for text generation, but token-by-token decoding usually leads to inefficiency in inference stage. Non-autoregressive Transformers (NATs)~\cite{nat2018gu, gu2018nonautoregressive, axe2020ghazvininejad, huang2022directed} significantly reduce the decoding latency by removing the dependency between target tokens and generating the whole sequence in parallel. 

Despite the fast decoding speed,  the main challenge of NATs lies in the learning of the ground-truth data distribution, which often has a large number of modalities \citep{gu2018nonautoregressive}.\footnote{We use the word \textit{multi-modality} in the context of machine translation literature, where it refers to to the phenomenon that there are multiple acceptable translations in the ground-truth data distribution and all of them are used as training labels for the same input.}
ATs mitigate this problem by treating sequence generation as a gradual modality collapsing process~\cite{gu2018nonautoregressive, qian2021glancing}. As the generation of later tokens is conditioned on the previous ones, it is unlikely for the model to flip around different modalities. 
 NATs, on the other hand, generate all the tokens all at once, hence prone to generate tokens from mixed modalities in one sequence, which strongly hurts their performance.

A common fix for this issue is to directly reduce the number of modalities of original data distribution by
knowledge distillation through an autoregressive model~\citep{gu2018nonautoregressive}. Intuitively, this step regenerates the training set using an autoregressive model learned from the original data distribution, making it more manageable for the NATs but also introducing a redundant pipeline.
Recent efforts start to scrutinize NATs' conditional independence assumption and address the multi-modality challenge mainly by learning target alignment~\citep{ctc2018libovicky, axe2020ghazvininejad, oaxe2021du, huang2022directed} and enhancing input~\citep{qian2021glancing, ghazvininejad2019mask}.

In this work, we propose to tackle this problem from a new perspective, with the goal of bypassing the modalities that are difficult for the model to capture in the learning procedure. Our method starts from  the reverse Kullback-Leibler (KL) divergence optimization, which effectively converges to a local single modality given a complex multi-modal target distribution~(\citealt{bishop2006pattern}; Figure~\ref{fig:tsne-exp}). 
The connection between modality learning and model distribution optimization is further explained in~\S\ref{sec:method-start}. To stabilize the training and prevent collapsing, we derive a set of constraints from the data distribution and impose them to the model distribution. We find that in theory leads to a contrastive learning objective (\S\ref{sec:cl-obj}). 
Finally, we show how to integrate this objective with the state-of-the-art NAT architecture DA-Transformer~\cite{huang2022directed} (\S\ref{sec:impl}).

We test the performance of our model, \method, on three generation tasks with five benchmarks. Experiments on two major WMT machine translation benchmarks demonstrate that our approach substantially improves the performance over strong baseline DA-Transformer and establishes new state-of-the-art results directly trained on the raw datasets~(\S\ref{sec:mt}). Similar to machine translation, we also achieve the best result for NAT on paraphrasing (\S\ref{sec:paraphrase}), and impressive results on non-autoregressive summarization (\S\ref{sec:summ}). \method exceeds the autoregressive model on two widely-used summarization benchmarks  XSum~\cite{narayan2018don} and Gigaword~\cite{rush2015neural}.

\begin{figure}
\centering
\begin{subfigure}{.30\textwidth}
    \begin{center}
        \includegraphics[width=\textwidth]{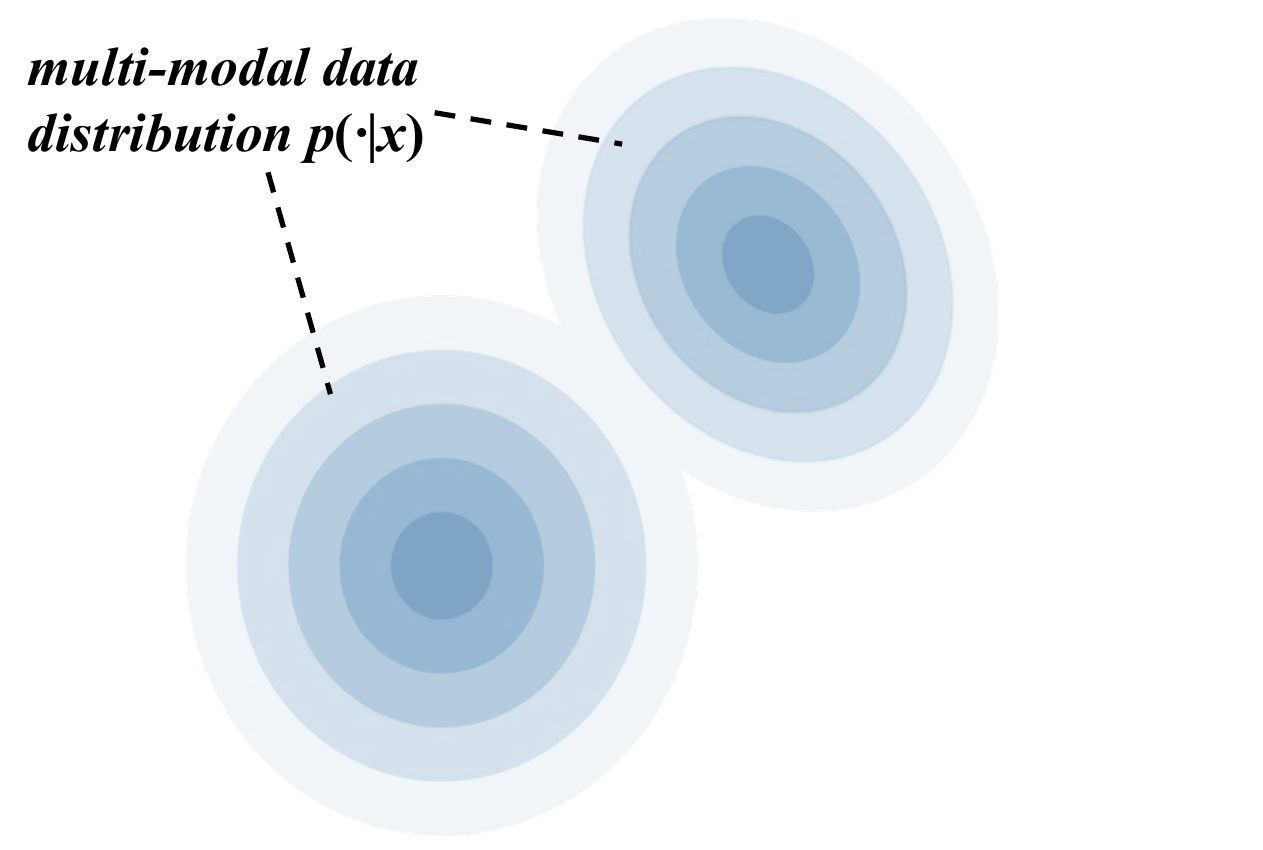}
        \vspace{-1em}
        \caption{}
        \label{fig:multimodality}
    \end{center}
\end{subfigure}
\hspace{-1mm}
\begin{subfigure}{.31\textwidth}
  \begin{center}
    \includegraphics[width=\textwidth]{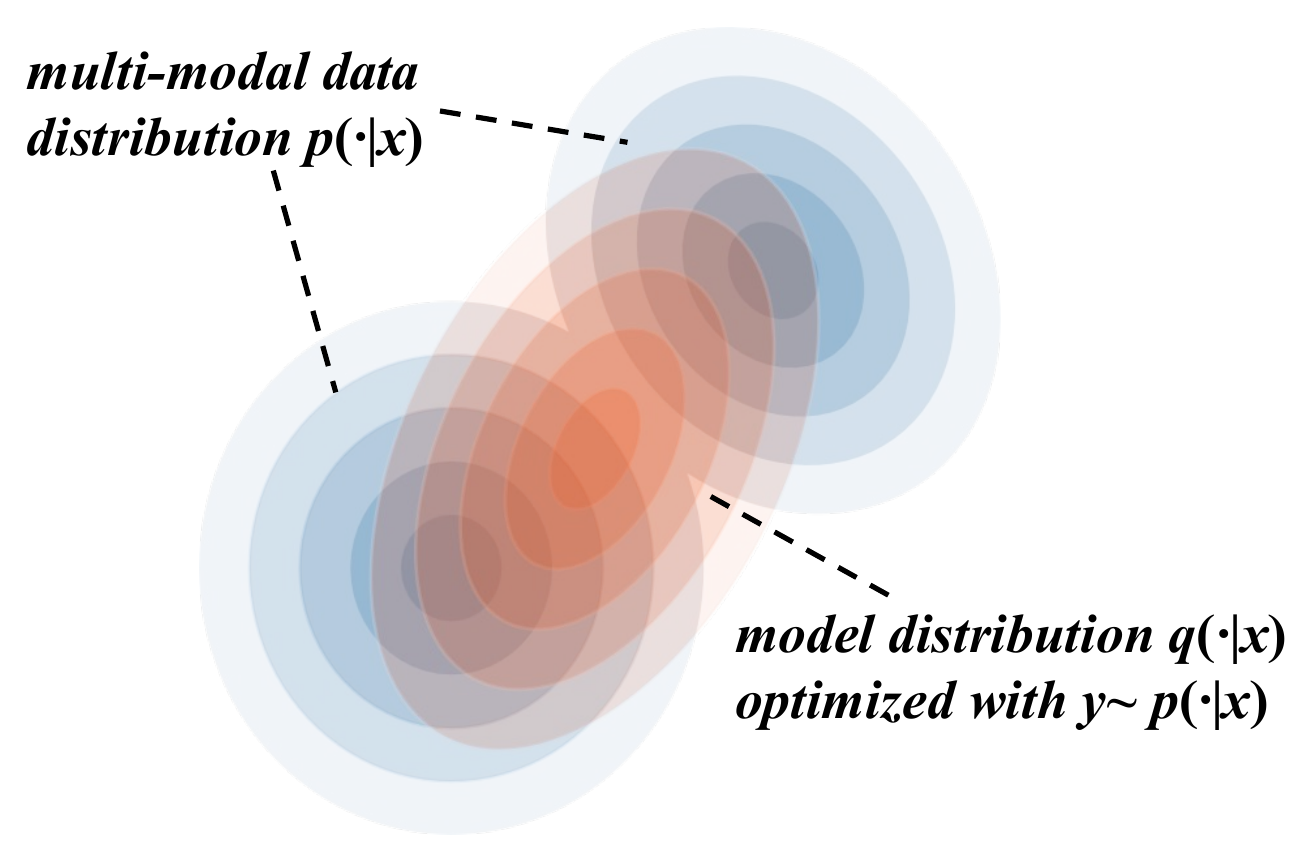}
    \vspace{-1em}
    \caption{}
    \label{fig:forward-kl}
  \end{center}
\end{subfigure}
\hspace{1mm}
\begin{subfigure}{.31\textwidth}
  \begin{center}
    \includegraphics[width=\textwidth]{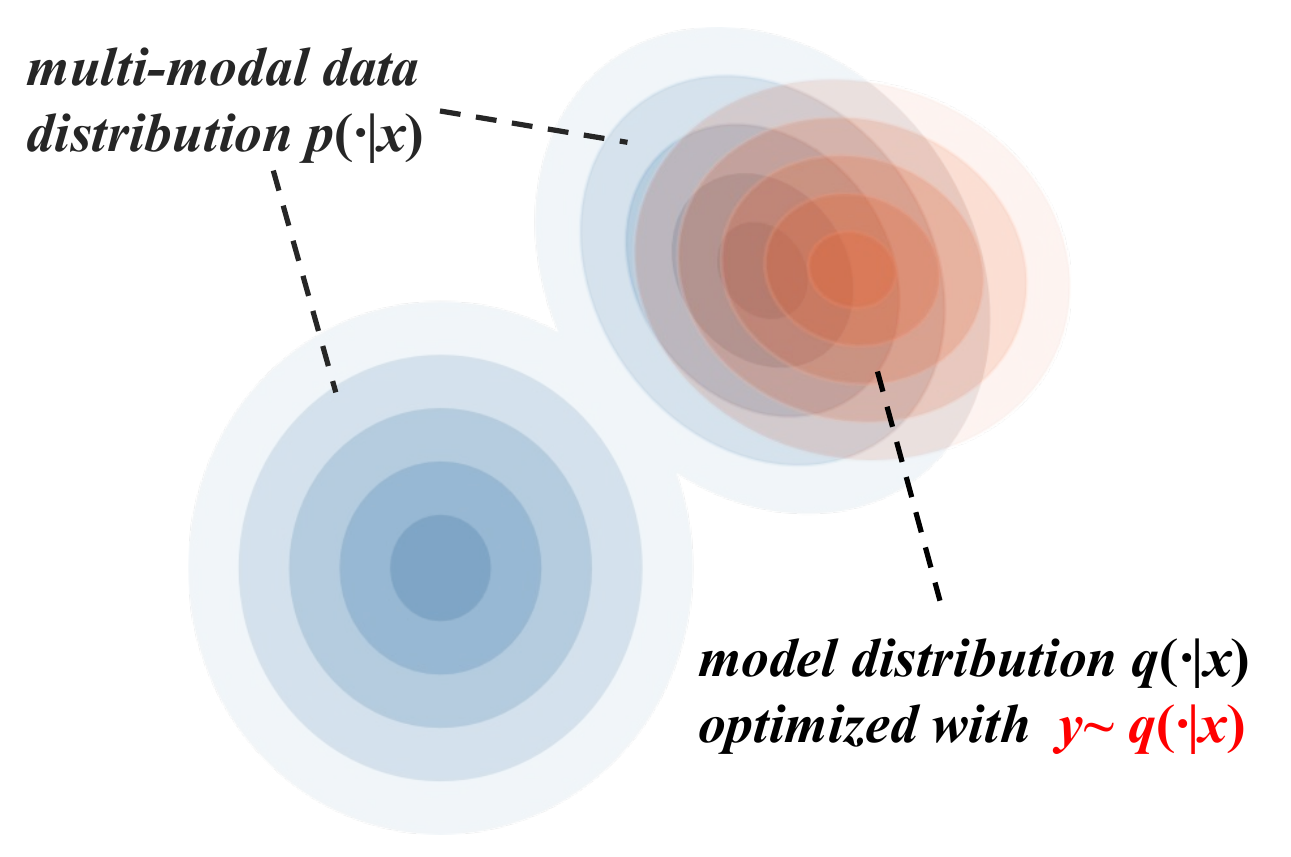}
    \vspace{-1em}
    \caption{}
    \label{fig:reverse-kl}
  \end{center}
\end{subfigure}
\caption{(a) shows a bimodal data distribution (blue contours). Orange contours in (b) show a single modal model distribution that is optimized to fit the data distribution by samples from the bimodal data distribution.  similar to (b), but the model distribution is optimized to a different local minimum with samples from the captured modality of model distribution. }
\label{fig:tsne-exp}
\end{figure}

%% file: 02_bg.tex
\section{Background}
\paragraph{Non-autoregressive Generation} Consider predicting a target sequence $\boldsymbol{y}=\{y_1,y_2,\ldots,y_n\}$ with a source sequence $\boldsymbol{x}=\{x_1,x_2,\ldots,x_m\}$, where $n$ and $m$ are target and source sequence length respectively.
Typical autoregressive transformers model the conditional probability $p(\boldsymbol{y}|\boldsymbol{x})$ via autoregressive decomposition as:
\begin{equation}
    p(\boldsymbol{y}|\boldsymbol{x}) = \prod_{i=1}^n p(y_i|y_{<i}, \boldsymbol{x}).
\end{equation}
While non-autoregressive transformers factorize $p(\boldsymbol{y}|\boldsymbol{x})$ independently with \textit{conditional independence assumption}:
\begin{equation}
    p(\boldsymbol{y}|\boldsymbol{x}) = \prod_{i=1}^n p(y_i|\boldsymbol{x}).
\end{equation}
With this assumption, NATs are able to drop the left-to-right dependencies and decode the target sequence in parallel.

 \paragraph{DP Training for NATs}\label{sec:dp}
The strict position alignment between predicted and target tokens~\cite{nat2018gu, axe2020ghazvininejad} of vanilla NAT has a poor ability to capture the multi-modality data distribution  which typically results in generated tokens from mixed modality and repeated words~\cite{gu2018nonautoregressive}. Dynamic Programming (DP) training~\cite{ctc2018libovicky, axe2020ghazvininejad, oaxe2021du} greatly alleviate this problem by introducing a long decoder length and alignment-based training objectives that marginalize all the possible alignments that yield the ground-truth. Take the latest work DAT (DA-Transformer)~\cite{huang2022directed} as an example, given the ground truth sequence $\boldsymbol{y} = \{y_1, y_2, \ldots, y_n\}$ whose length is $n$ and decoder length $L \gg n$, $ \mathrm{log} \, q(\boldsymbol{y}|\boldsymbol{x})$ can be defined as:
\begin{equation}
\label{eq:dp-loss}
\begin{split}
    \mathrm{log} \, q(\boldsymbol{y}|\boldsymbol{x}) &=\mathrm{log}\, \sum_{\mathbf{a} \in \Gamma } q(\boldsymbol{y}, \mathbf{a}|\boldsymbol{x})=\mathrm{log}\, \sum_{\mathbf{a} \in \Gamma } q(\boldsymbol{y}|\mathbf{a}, \boldsymbol{x})\cdot q(\mathbf{a} | \boldsymbol{x}),\\
    q(\mathbf{a} | \boldsymbol{x}) &= \prod_{i=1}^{n-1} \mathbf{E}_{a_i. a_{i+1}},\\
    q(\boldsymbol{y}|\mathbf{a}, \boldsymbol{x}) &= \prod_{i=1}^n \mathrm{softmax}(\mathbf{W}_{\text{vocab}}\mathbf{h}_{a_i}),
\end{split}
\end{equation}
where $\mathbf{a}$ is a set of decoder position indexes sorted in ascending order whose size $|\mathbf{a}|=n$ and  $\Gamma$ contains all possible $\mathbf{a}$ with a size of $\binom{L}{n}$. For example,  target length $n=3$ and $L=6$ means $\Gamma$ contains $\binom{6}{3}=20$ possible $\mathbf{a}$, and $\mathbf{a} \in \{0,1,2\}, \{0,1,3\} ... \{3,4,5\}$.  $\mathbf{h}_{a_i}$ means the $a_i$-th decoder hidden state. $q(\boldsymbol{y}|\mathbf{a}, \boldsymbol{x})$ is the token prediction probablity in generation models. $q(\mathbf{a} | \boldsymbol{x})$ is given by the transition matrix $ \mathbf{E} \in \mathbb{R}^{L\times L}$ modeling the first-order dependency between decoder position indexes where $\mathbf{E}_{a_i. a_{i+1}}$ means the transition probability of index $a_i$ to index $a_j$ and it is predicted based on the decoder hidden states. Enumerating all $\mathbf{a} \in \Gamma$ will result in the curse of combination, but luckily $\mathbf{E}$ can be trained with the dynamic programming algorithm. Details of the algorithm can be found in \citet{huang2022directed}.  In Figure~\ref{fig:dp_vanilla} Appendix, we present an example that highlights 1) the difference between a Dynamic Programming (DP) model and a Vanilla model; 2) the calculation of $q(\mathbf{a} | \boldsymbol{x})$ and  $q(\boldsymbol{y}|\mathbf{a}, \boldsymbol{x})$. For the decoding procedure, please refer to Appendix~\ref{sec:codat}. Training with Dynamic Programming is essential for NATs on raw data; otherwise, their performance lags notably behind that of ATs.

%% file: 03_method.tex
\section{Method}

\subsection{Learning on Model Distribution}\label{sec:method-start}

Unlike previous efforts that mitigate multi-modality problems via architecture design or decoding algorithm, we attribute the multi-modality problem to the learning objective, maximum likelihood estimation (MLE). MLE minimizes the KL divergence between the data distribution $p(\cdot|\boldsymbol{x})$ and the model distribution $q(\cdot|\boldsymbol{x})$, which can be formulated as follows,
\begin{equation}
\label{eq:forward-kl}
\begin{split}
    \text{KL}(p \parallel q)&=\mathbb{E}_{\boldsymbol{y}\sim p(\cdot|\boldsymbol{x})}[\log \frac{p(\boldsymbol{y}|\boldsymbol{x})}{q(\boldsymbol{y}|\boldsymbol{x})}] \\
    &=-H(p)+\mathbb{E}_{\boldsymbol{y}\sim p(\cdot|\boldsymbol{x})}[- \log q(\boldsymbol{y}|\boldsymbol{x})]
\end{split}
\end{equation}
where $H(p)$ denotes the entropy of data distribution and is a constant.
Note that MLE requires sampling from the data distribution $p(\cdot|\boldsymbol{x})$, which forces the NAT model to assign a reasonable probability for each sampled $\boldsymbol{y}$. We argue that this objective could be too difficult for NAT models because of its limited capacity, which may result in low-quality predictions with mixed modalities.

In this work, we suggest that the above problem can be alleviated by replacing the samples from $p(\cdot|\boldsymbol{x})$ with the samples from $q(\cdot|\boldsymbol{x})$, therefore dropping some modalities that are too difficult for NATs to capture. Specifically, we propose a generalized divergence $D(q \parallel p)$ as our objective:
\begin{equation}
\label{eq:reverse-kl}
\begin{split}
D(q \parallel p)&=\mathbb{E}_{\boldsymbol{y}'\sim q(\cdot|\boldsymbol{x})}[M_{p,q}(\boldsymbol{y}'|\boldsymbol{x})],
\end{split}
\end{equation}
where $M_{p,q}(\boldsymbol{y}'|\boldsymbol{x})$ measures the discrepancy between model and data distribution with a given sample $\boldsymbol{y}'$ from the simpler model distribution.

To better explain our intuition, we illustrate an example in Figure~\ref{fig:tsne-exp}. In previous objective, $\mathbb{E}_{y\sim p(\cdot|\boldsymbol{x})}[\cdot]$ optimizes samples from all modalities in data distribution $p(\cdot|\boldsymbol{x})$ by assigning them importance rates $p(\boldsymbol{y}|\boldsymbol{x})$, thereby leading to a mixed-modal distribution~(see Figure~\ref{fig:forward-kl});
in contrast, $\mathbb{E}_{y'\sim q(\cdot|\boldsymbol{x})}[\cdot]$ focuses on samples within a captured modality from modeling distribution $q(\cdot|\boldsymbol{x})$ by assigning low importance rate to uncaptured modalities and keeps on optimizing the captured modalities~(see Figure~\ref{fig:reverse-kl}).  The final training objective is the sum of $\text{KL}(p \parallel q)$ and $\text{D}(q \parallel p)$. $\text{KL}(p \parallel q)$ acts as a regularization to prevent collapse if the samples from model distribution are low-quality.

Our idea is connected to previous work that optimizes reverse KL based on reinforcement learning~\footnote{Detailed discussion on reverse KL and reinforcement learning are presented in Appendix~\ref{app:rkl-rl}.}~\citep{wu2016google, bahdanau2017an}. Note that the reverse $\text{KL}(q \parallel p)$ is a special case of Eq.~\ref{eq:reverse-kl}, which adopts the discrepancy measure as $\log \frac{q(\boldsymbol{y}'|\boldsymbol{x})}{p(\boldsymbol{y}'|\boldsymbol{x})}$. However, reinforcement learning  necessitates a well-structured parameter space for initialization~\cite{choshen2019weaknesses}, but in our preliminary experiments, the original KL loss optimized by DP training is easily perturbed by the RL loss that it ultimately impedes performance improvements (see Figure~\ref{fig:rl_dp}).  In the next section, we bypass the difficulty of directly optimizing the absolute reward distribution by suggesting more flexible necessary conditions and optimizing the proposed objective  contrastively.

\subsection{A Contrastive Learning Objective}\label{sec:cl-obj}
Having the connection between NATs' multi-modality challenge and model distribution optimization in the previous discussion, we focus on how to concretize a discrepancy measure $M_{p,q}(\boldsymbol{y}'|\boldsymbol{x})$, in a contrastive way, to NAT models. The general methodology is investigating the constraints that are sufficient and necessary to data distribution and then imposing these constraints to model distribution to guide the learning process.

With $\boldsymbol{y}'$ sampled from model distribution, it is usually unobserved in the enormous target space, with high probability, in a dataset generated by an unknown distribution.
Thus it is untractable to quantify the likelihood $p(\boldsymbol{y}'|\boldsymbol{x})$.
A practical estimation is to introduce a reward-based distribution $p^R(\cdot|\boldsymbol{x},\boldsymbol{y})$, i.e.,
\begin{equation}
\label{eq:pr}
p^R(\boldsymbol{y}'|\boldsymbol{x},\boldsymbol{y})=\frac{1}{Z^R}\exp (R(\boldsymbol{y},\boldsymbol{y}')), (\boldsymbol{x},\boldsymbol{y})\sim \mathcal{D},
\end{equation}
where $R(\cdot, \cdot)$ is a reward function measuring $(\boldsymbol{y},\boldsymbol{y}')$'s divergence, and $Z^R$ denotes the normalizer.
We here use BLEU~\citep{papineni2002bleu} as a lexical measure for optimization.

We then seek a series of contrastive conditions that are sufficient and necessary to generate the data distribution $p^R(\cdot|\boldsymbol{x}, \boldsymbol{y})$:
\begin{equation}
\label{eq:nsc}
\begin{split}
    \forall \boldsymbol{y}'_+,&\boldsymbol{y}'_-: \log \frac{p^R(\boldsymbol{y}'_+|\boldsymbol{x},\boldsymbol{y})}{p^R(\boldsymbol{y}'_-|\boldsymbol{x},\boldsymbol{y})} = \epsilon(\boldsymbol{y}'_+,\boldsymbol{y}'_-|\boldsymbol{y}) \\
    &\epsilon(\boldsymbol{y}'_+,\boldsymbol{y}'_-|\boldsymbol{y})= R(\boldsymbol{y}'_+,\boldsymbol{y})-R(\boldsymbol{y}'_-,\boldsymbol{y}),
\end{split}
\end{equation}
where $\boldsymbol{y}'_+$ and $\boldsymbol{y}'_-$ are two samples in target space satisfying $R(\boldsymbol{y}'_+,\boldsymbol{y})>R(\boldsymbol{y}'_-,\boldsymbol{y})$, without loss of generality, and $\epsilon(\boldsymbol{y}'_+,\boldsymbol{y}'_-|\boldsymbol{y})$ represents $(\boldsymbol{y}'_+,\boldsymbol{y}'_-)$'s reward gap.
Detailed proof of necessity and sufficiency is shown in Appendix~\ref{app:proof-sufficient-necessary}.
We here focus on its necessity, which is treated as a constraint to model distribution in the future. 
Consider a bundle of generated sequences $\{\boldsymbol{y}'_k\sim q(\cdot|\boldsymbol{x})\}_{k=1}^K$, satisfying $\forall i>j: R(\boldsymbol{y}'_i,\boldsymbol{y})>R(\boldsymbol{y}'_j,\boldsymbol{y})$, where $K$ is the number of samples.
Assuming a small positive lower bound $\epsilon_\text{LB}$, namely $\forall i>j: \epsilon(\boldsymbol{y}'_i,\boldsymbol{y}'_j|\boldsymbol{y}) \ge \epsilon_\text{LB}$, we have looser necessary conditions pairwisely, as
\begin{equation}
\label{eq:loose-constraints}
    \forall i>j: \log \frac{p^R(\boldsymbol{y}'_i|\boldsymbol{x},\boldsymbol{y})}{p^R(\boldsymbol{y}'_j|\boldsymbol{x},\boldsymbol{y})} \ge \epsilon_\text{LB},
\end{equation}
where $\epsilon_\text{LB}$ is treated as a hyper-parameter.
We neglect the event $R(\boldsymbol{y}'_i,\boldsymbol{y})=R(\boldsymbol{y}'_j,\boldsymbol{y})$ whose probability is negligible.
By taking $\{\boldsymbol{y}'_{j+1},\ldots, \boldsymbol{y}'_{i-1}\}$ as intermediate states, we derive stronger necessary conditions,
\begin{equation}
\label{eq:ranking-constraints}
    \forall i>j: \log \frac{p^R(\boldsymbol{y}'_i|\boldsymbol{x},\boldsymbol{y})}{p^R(\boldsymbol{y}'_j|\boldsymbol{x},\boldsymbol{y})} \ge (i-j)\epsilon_\text{LB},
\end{equation}
which is also a ranking-based condition.
After all, we impose the conditions in Eq.~\ref{eq:ranking-constraints} as constraints to model distribution $q(\cdot|\boldsymbol{x})$, and penalize $q(\cdot|\boldsymbol{x})$ voilating the constraints:
\begin{equation}
\label{eq:final-constraints}
    \forall i>j: \mathcal{L}_{i,j}=\max \{0, -\log {q(\boldsymbol{y}'_i|\boldsymbol{x})} +\log{q(\boldsymbol{y}'_j|\boldsymbol{x})} + (i-j)\epsilon_\text{LB}\}.
\end{equation}
In typical decomposition-based sequence modeling, $q(\cdot|\boldsymbol{x})$ is heavily biased by sequence length, where shorter ones obtain higher probability and are more favorable~\citep{wu2016google}.
To eliminate the length bias, we normalize $q(\cdot|\boldsymbol{x})$ by the target length $||\boldsymbol{y}'||$. In experiments, we implement Eq.~\ref{eq:final-constraints} by: 1) sampling $K$ hypotheses from the model distribution, 2) constructing $\binom{K}{2}$ pairs, and 3) calculating $\mathcal{L}_{i,j}$ for each pair.   Eq.~\ref{eq:final-constraints} is a contrastive objective, facilitating $M_{p,q}(\cdot|\boldsymbol{x})$ in Eq.~\ref{eq:reverse-kl} by penalizing violated constraints and we show that Eq.~\ref{eq:final-constraints} can be written with similar form with Eq.~\ref{eq:reverse-kl} in Appendix~\ref{app:contrastive}, yet it's crucial to note that they are not strictly equivalent.
By optimizing with the contrastive constraints, we show that the original regularization (DP) loss is less likely to obviously increase in Figure~\ref{fig:cl_dp}.

\subsection{Implementation on NATs}\label{sec:impl}
Practical implementation of the contrastive loss includes sampling multiple hypotheses from the model distribution and ranking the likelihood of generated sequences according to some evaluation metrics such as BLEU.  
The sampling algorithm can be Noisy Parallel Decoding~\cite{cho2016noisy, gu2018nonautoregressive} on vanilla NATs which randomly add Gaussian noise to the word embedding layer to generate diverse output,  but it involves $sampling\_size$ times additional forward passes consequently leading to increased training costs. Results on vanilla NATs are shown in Table~\ref{tab:glat}.
As for DP-based models, sampling positive-negative examples do not obviously bring training cost since the sampling process can be viewed as combining different decoder predictions (See Figure~\ref{fig:dp_vanilla} in Appendix) and can be typically implemented by beam search or nucleus sampling~\cite{holtzman2019curious}. 

We select previous the best DP-based model DA-Transformer and its pretrained version (for summarization tasks) as our base model in the main experiments considering the following two advantages. i) \textbf{Accuracy}: Samples used to minimize the divergence in Eq.\ref{eq:reverse-kl} are expected to be high-quality and single-modal. In other words, if all training samples are of low quality, the model will be tuned in the wrong direction. Basing on SOTA provides a high-quality sampling space that stabilizes the optimization process and benefits the performance. ii) \textbf{Efficiency}:  DA-Transformer is a DP-based model which means sampling multiple hypotheses usually does not involve more cost with the sample size increasing.
As mentioned before, DA-Transformer and most DP-based models usually have an extended decoder length to relax the strict alignment between decoder prediction and target words. Actually, sampling hypotheses from DP-based models is equivalent to combing the predicted words from different decoder positions, which means a generated sequence (length<$L$) is a subset of words from all $L$ predicted words (see the example in Figure~\ref{fig:dp_vanilla} Appendix) without more floating point operations or forward passes (see Appendix~\ref{sec:codat}). 

In experiments, we find the process of contrasting low-likelihood sequences adversely has a negative effect on the original DP training loss, resulting in instability and unsatisfied results during training. In order to avoid training on such samples, we adopt two filtering tricks: 1) hypotheses-level filtering: sampling a larger amount of hypotheses first and then keeping only the top 25\% for training.  2) Samples-level filtering:  some challenging training samples may exhibit significantly low sequence likelihood on target from data distribution and usually have low-quality sampling space. Persisting in optimizing on these samples with model generated sequences will ultimately result in poor performance (in both reinforcement learning and contrastive learning setting). Therefore, we exclude training samples with DP training loss exceeding $\alpha$ when optimizing the discrepancy $M_{p,q}(\boldsymbol{y}'|\boldsymbol{x})$ with target sequences from the model distribution. $\alpha$ is practically set to the value of $\mathbb{E}_{\boldsymbol{y}\sim \mathcal{V}}[- \log q(\boldsymbol{y}|\boldsymbol{x})]$ where $\mathcal{V}$ means the validation set.

%% file: 04_exps.tex
\section{Experiments}
We verify \method on three different text generation tasks, including machine translation, summarization, and paraphrasing, with five different benchmarks. We show that the results of non-autoregressive Transformers can be boosted to a new level with the contrastive learning framework. Particularly, we exceed the autoregressive Transformer on the two summarization datasets and achieve an improvement of \textbf{0.83} BLEU score on average on the translation datasets and paraphrase dataset. Training details and  hyperparameters  can be found in Appendix~\ref{sec:setup}

\subsection{Quantitative Results}
This section shows the results of \method on serval mainstream text generation benchmarks. To prove our method has a stronger modality learning ability, we directly use the \textbf{raw} dataset for machine translation and paraphrasing. We use distillation data for summarization following previous work. We use the most standard evaluation metrics for each task in our main results.

\paragraph{Machine Translation}\label{sec:mt}
For machine translation, we evaluate \method on WMT14 English $\leftrightarrow$ German translation task (4M samples), WMT17 English $\leftrightarrow$ Chinese translation task (20M samples) and follow the default split of test and validation set of FairSeq~\cite{ott-etal-2019-fairseq}. We evaluate our model on all datasets with tokenized BLEU except WMT17 En$\rightarrow$Zh, which is assessed with SacreBLEU~\cite{sacrebleu2018}. We also reproduce the  autoregressive Transformer and DA-Transformer (DAT) with the same running environment as \method. Since previously reported results is trained with 100k steps and a batch size of 32k tokens, our implementation achieves higher performance. The main results are shown in Table~\ref{tab:nmt-resutls}.  Among all iterative NATs and fully NATs, our method \method achieves the best result and outperforms previous state-of-the-art by a remarkable margin.  The average performance gap between the autoregressive transformer and non-autoregressive transformer is reduced to 0.4 with our contrastive learning method. We also show that directly optimizing the reward can also bring improvements to the original DAT. However, due to its negative impact on the original training loss, optimizing the constraints usually brings higher performance.  All the hypotheses are sampled from the model distribution in parallel without normalizing and selecting at each step like the beam search algorithm. Therefore, it still maintains a 10.2 times speedup during inference with only native PyTorch operations. 
\renewcommand\arraystretch{0.95}
\begin{table*}[t]
\begin{center}
\vspace{-1em}
\caption{BLEU scores on WMT14 En$\leftrightarrow$De, WMT17 Zh$\leftrightarrow$En translation tasks. The results of previous NATs are quoted from their papers and \citet{huang2022directed}. +Reward means we optimize the reward distribution derived from Eq.~\ref{eq:reverse-kl} and +Constraints means we optimize the contrastive objective.  The best results for NATs are in \textbf{bold}.
}
\label{tab:nmt-resutls}
\tabcolsep0.1 in
\begin{tabular}{lccccccc}
\toprule
\multirow{2}{*}{\vspace{-2mm}\bf Model} & \multirow{2}{*}{\bf  Iter.} & \multirow{2}{*}{\bf  Speedup} & \multicolumn{2}{c}{\bf WMT14} &\multicolumn{2}{c}{\bf WMT17}   \\
& & & \bf En-De & \bf De-En &\bf En-Zh  &\bf Zh-En \\ 
\midrule
Transformer~\cite{transformer2017vaswani}& $N$ & 1.0x & 27.6 &31.4 & 34.3 & 23.7\\
Transformer (Ours)  & $N$ & 1.0x & 27.85 & 31.64 & 34.89 & 23.72\\
\midrule
CMLM~\cite{cmlm2019ghazvininejad} & 10 & 2.2x & 24.61 & 29.40 & -- & -- \\
SMART~\cite{smart2020}  & 10 & 2.2x & 25.10 & 29.58 & -- & -- \\
DisCo~\cite{disco2020kasai}  & $\approx	4$ & 3.5x & 25.64 & -- & -- & -- \\
Imputer~\cite{imputermt2020saharia} & 8 & 2.7x & 25.0 & -- & -- & --\\
CMLMC~\cite{cmlmc2021}  & 10 & 1.7x & 26.40 & 30.92 & -- & --\\

\midrule
Vanilla NAT~\cite{nat2018gu} & 1 & 15.3x & 11.79 & 16.27 & 18.92 & 8.69\\
AXE ~\cite{axe2020ghazvininejad}  & 1 & 14.2x & 20.40 & 24.90 & -- & --\\
OaXE~\cite{oaxe2021du} & 1 & 14.2x & 22.4 & 26.8 & -- & -- \\
CTC~\cite{ctc2018libovicky} & 1 & 14.6x & 18.42 & 23.65 & 26.84 & 12.23\\
GLAT~\cite{qian2021glancing}  & 1 & 15.3x & 19.42 & 26.51 & 29.79 & 18.88\\
CTC + GLAT~\cite{qian2021glancing}  & 1 & 14.2x & 25.02 & 29.14 & 30.65 & 19.92\\
DAT~\cite{huang2022directed}  & 1 & 13.9x & 26.57 & 30.68 & 33.83 & 22.82\\
DAT+Viterbi~\cite{shao2022rephrasing}  & 1 & 13.2x & 26.89 & 31.10 & 33.65 & 23.24\\
\midrule
DAT (Ours)  & 1 & 13.9x & 26.63 & 30.73 & 33.56 & 22.68\\
\quad + Reward  & 1 & 10.2x & {27.06} & {31.22} & {33.59} & {23.16} \\
\quad+ Constraints   & 1 & 10.2x & \textbf{27.40} & \textbf{31.64} & \textbf{34.23} & \textbf{23.71} \\
\bottomrule
\end{tabular}
\end{center}
\end{table*}

\paragraph{Paraphrase}\label{sec:paraphrase}
For paraphrase, we adopt the Quora Question Pairs (QQP)\footnote{\url{https://www.kaggle.com/c/quora-question-pairs}} collected from the Quora community question answering forum with 147 training pairs. We evaluate generated results with BLEU and ROUGE-L. The performance gap between \method and the autoregressive baseline is reduced to 0.02 BLEU and we even get 0.6 improvements over the autoregressive model in terms of ROUGE-L.
\begin{figure}[!t]
\begin{minipage}[t]{0.53\textwidth}
\centering
\captionof{table}{{Results on the WMT14 en-de benchmark of vanilla model GLAT and DP-based model DAT with different training objectives: MLE, Reward, and Constraints. For GLAT, directly optimizing the reward distribution also yields satisfactory results, but for DAT, optimizing using constraints usually leads to better performance. We use the NPD size = 7.}
}
\label{tab:glat}
\resizebox{1.0\columnwidth}{!}{%
\setlength\tabcolsep{7pt}
\begin{tabular}{lccc}
\toprule
{\textbf{Model}} & \textbf{Dataset} & \textbf{Objective} & \textbf{BLEU}   \\
\midrule
GLAT & Raw  & $\text{D}(p \parallel q)$ & 19.42 \\
GLAT & KD  & $\text{D}(p \parallel q)$ & 25.28 \\
GLAT & KD  & $\text{D}(q \parallel p)$ (Reward) & \textbf{25.71} \\
GLAT & KD  & $\text{D}(q \parallel p)$(Constraints) & 25.66 \\
\midrule
DAT & Raw & $\text{D}(p \parallel q)$ & 26.63  \\
DAT & Raw & $\text{D}(q \parallel p)$ (Reward) & 27.06  \\
DAT & Raw & $\text{D}(q \parallel p)$(Constraints) & \textbf{27.40}  \\
\bottomrule
\end{tabular}}
\vspace{1em}

\label{tab:paraphrase}
\vspace{-1em}
\end{minipage}
\hfill
\begin{minipage}[t]{0.45\textwidth}
\centering
\captionof{table}{{Results on the test set of QQP in terms of BLEU and ROUGE-L. Results with $*$ are from \citet{gong2022diffuseq}. The best results of NAT are in bold. \method has $\sim$6.8$\times$ speedup compared with the autoregressive model.}
}
\label{tab:code-results}
\resizebox{1.0\columnwidth}{!}{%
\setlength\tabcolsep{6pt}
\begin{tabular}{lcc}
\toprule
{\textbf{Model}} & \textbf{BLEU} & \textbf{ROUGE-L}   \\
\midrule
GRU-attention$^*$ & 18.94 &51.29  \\
Transformer & 27.80 & 57.94  \\
\midrule
LevT~\cite{DBLP:journals/corr/abs-1905-11006}$^*$  & 22.68 & 57.95 \\
GLAT~\cite{qian2021glancing} & 25.16 & 58.28 \\
DiffuSeq~\cite{gong2022diffuseq}$^*$  & 24.13 & \textbf{58.80} \\
DAT~\cite{huang2022directed} & 27.00 & 58.31  \\
\midrule
\method &\textbf{27.78} & {58.54} \\
\bottomrule
\end{tabular}}
\end{minipage}
\end{figure}

\paragraph{Summarization}\label{sec:summ}
For summarization, we achieve new state-of-the-art for NATs on two  widely-used news summarization datasets XSum~\cite{narayan2018don} and Gigaword~\cite{rush2015neural}.  Gigaword has 4M article-headline pairs and XSum has 227k articles with human-written summaries from BBC News. We use ROUGE~\cite{lin2004rouge} as the evaluation metric.
To have a fair comparison with previous work, we also perform knowledge distillation before training. Table~\ref{tab:summary-results} shows the ROUGE scores on Summarization datasets. Similar to machine translation,  we also implement an autoregressive Transformer with the same running environment e.g., the same knowledge distillation dataset and training steps. \method surpasses the autoregressive baselines on both XSum and Gigaword.  Generally, \method improves DAT by more than 0.5 ROUGE-1 score and improves the pretrained version of DAT by 1.07 ROUGE-1 score.  On XSum, \method without pretraining obtains on-par performance with BANG~\cite{qi2021bang} which is pretrained on large-scale unlabeled corpus and exceeds BANG by 2.44 ROUGE-2. \method achieves a speedup ratio of  $\sim 8.4 \times$  on XSum and  $\sim 6.0 \times$ on Gigaword compared with the autoregressive basline. In addition, we also validate our method by building on the strong pretrained version of DAT~\cite{huang2023directed}. Results show that \method (pretrained) even exceeds many strong pretrained AT baselines on XSum.
\renewcommand\arraystretch{0.95}
\begin{table*}[t]
\begin{center}
\caption{ROUGE scores on Summarization datasets. \textit{Avg} means the average of ROUGE-1, ROUGE-2 and ROUGE-L.  Results with $^\dag$ are token from~\citet{su2021non} and results with $^*$ are from~\citet{qi2021bang}. Results that exceed the autoregressive model are underlined (w/o pretrained), and the best results of NAT are in bold. \method accelerates the autoregressive model by $\sim$8.4  times on XSum and  $\sim$6.0 times on Gigaword. The third block contains results of pretrained baselines.}
\tabcolsep0.1 in
\begin{tabular}{lccccccccc}
\toprule
\multirow{2}{*}{\vspace{-2mm}\bf Model} & \multicolumn{4}{c}{\bf XSum} & \multicolumn{4}{c}{\bf Gigaword}   \\
& \bf R-1 & \bf R-2 &\bf  R-L & \textit{Avg} &  \bf R-1 & \bf R-2 &\bf  R-L & \textit{Avg} \\ 
\midrule
Transformer~\cite{transformer2017vaswani} & 30.66 & 10.80 & 24.48 & 22.0 & 35.74 & 16.97 & 33.43 & 28.7 \\
Transformer (Ours)  & 32.07 & 11.31 & 25.96 & 23.1 & 36.99 & 18.22 & 34.48 & 29.9\\
\midrule
Vanilla NAT~\cite{nat2018gu}$^*$ & 24.04 & 3.88 & 20.32 & 16.1 & 27.20 & 8.96 & 25.58 & 20.6\\
LevT~\cite{DBLP:journals/corr/abs-1905-11006}  & 24.75 & 4.18 & 20.87 & 16.6 & 30.40 & 11.89 &  29.56 & 23.9\\
CMLM~\cite{cmlmc2021}$^*$  & 23.82 & 3.60 & 20.15 & 15.8 & -- & -- &  -- & -- \\
NART-CRF~\cite{sun2019fast}$^\dag$ & -- & -- & -- & -- &30.29 & 12.61 & 28.71 & 23.9\\
BERT-NAG~\cite{su2021non}$^\dag$& -- & -- & -- & -- & 35.05 & 16.48 & 33.28 & 28.3\\
DAT~\cite{huang2022directed} &31.87 & 11.02 & 25.93 & 22.9 & 36.52 & 17.94 & 34.30 & 29.6\\
\midrule
BANG~\cite{qi2021bang}$^*$ & {{32.59}} & 8.98 & {{27.41}} & 23.0 & -- & -- & -- & -- \\
MIST~\cite{jiang2021improving} & 34.63 & 11.29 & 28.70 & 24.9 & -- & -- & -- & --  \\
DAT (Pretrained)~\cite{huang2023directed} & 38.80 & 16.07 & 31.78 & 28.9 & 36.15 & 17.49 & 33.84 & 29.2\\
\midrule
\method & \underline{32.45} & {\underline{11.42}} & \underline{26.30} & \underline{23.4} & \textbf{\underline{37.01}} & \textbf{\underline{18.68}} & \textbf{\underline{34.63}} & \textbf{\underline{30.1}} \\
\method (Pretrained) & \textbf{39.87} & \textbf{{17.38}} & \textbf{32.61} & \textbf{{30.0}} & {36.92} & {18.23} & {34.46} & 29.9 \\
\bottomrule
\end{tabular} 

\label{tab:summary-results}
\end{center}
\end{table*}

\subsection{Discussion}



\renewcommand\arraystretch{0.4}
\begin{table*}[h!]
    \small
    \centering
\begin{tabular}{@{}c  p{0.7\textwidth}}
\toprule
\multicolumn{1}{l}{\bf Ground truth:} & According to the proposal, parking fees are to be increased by 50 percent.\\ 
\cmidrule{1-2}
\multirow{3}{*}{ 3 outputs from \bf DAT} & \textbf{1.} There is an increase in \textcolor{myOran}{parking charges} \textcolor{myBlue}{should be} increased by 50 \%. \\ 
\cmidrule{2-2}
     & \textbf{2.} The proposed increase in \textcolor{myOran}{parking charges} \textcolor{myBlue}{should therefore} be increased by 50 \%.\\
\cmidrule{2-2}
      & \textbf{3.} Accordingly, the increase in \textcolor{myOran}{parking charges}  \textcolor{myBlue}{are to be} increased by 50 \%. \\ 
\cmidrule{1-2}
\multirow{3}{*}{ 3 outputs from \bf CoDAT} & \textbf{1.}The proposal is that parking charges should therefore be increased by 50 \%. \\ 
\cmidrule{2-2}
     & \textbf{2.} According to it, the parking charges should be increased by 50\% .\\
\cmidrule{2-2}
      & \textbf{3.} According to it, there will  be increase in \textcolor{myOran}{parking charges}  \textcolor{myBlue}{should be} increased by 50 \% . \\ 
\bottomrule
\end{tabular}
\caption{ An example from WMT14 De$\rightarrow$En translation. We adopt a Sampling size = 128 and present the top 3 hypotheses from DAT and CoDAT. Given the input German  $\boldsymbol{x} = $"\texttt{Demnach sollen die Parkgebühren um 50 Prozent erhöht werden .}"   Though, the first-order dependency introduced by DAT effectively reduces repeated words, it still can not prevent the generated sequence mixing words from several possible translations while the rank of single-modal translation is higher in the sampling results of \method }
\label{tab:wmt-example}
\end{table*}

\begin{figure}
\centering
\begin{subfigure}{.46\textwidth}
    \begin{center}
        \includegraphics[width=\textwidth]{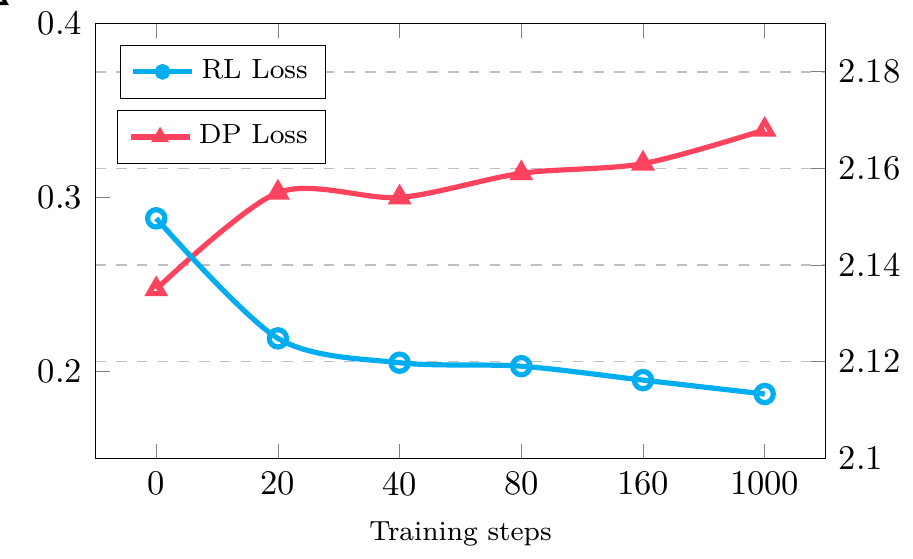}
\caption{ Reinforcement Learning (RL) loss and DP Loss on the validation set of WMT14 en-de benchmark. The left Y-axis represents the value of RL loss, while the right Y-axis corresponds to the value of DP loss. }
\label{fig:rl_dp}
    \end{center}
\end{subfigure}
\hspace{4mm}
\begin{subfigure}{.46\textwidth}
  \begin{center}
    \includegraphics[width=\textwidth]{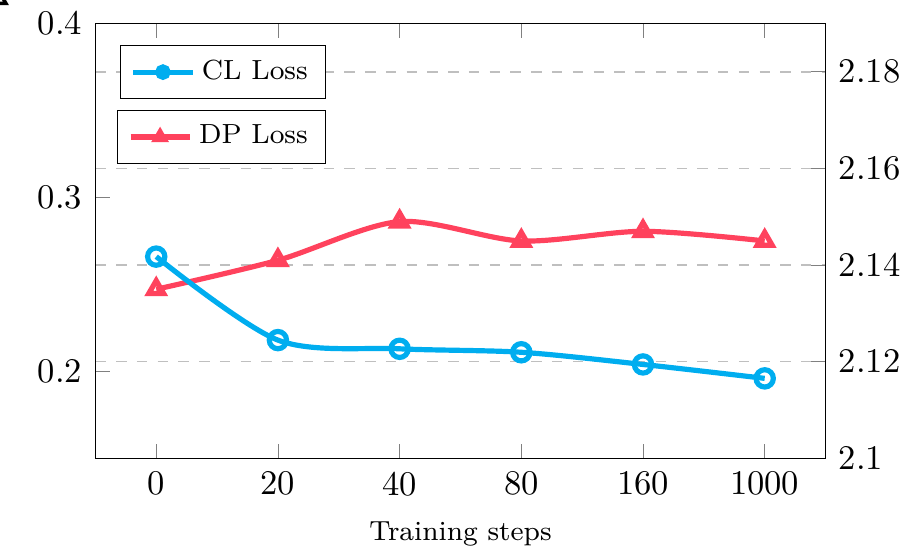}
        \caption{ Contrastive Learning (CL) loss and DP Loss on the validation set of WMT14 en-de benchmark. The left Y-axis represents the value of CL loss, while the right Y-axis corresponds to the value of DP loss. }
    \label{fig:cl_dp}
  \end{center}
\end{subfigure}
\end{figure}
\paragraph{Reducing Data Modality}
Normalized Corpus-level Multimodality(NCM)~\cite{sun2020approach} is a metric for measuring corpus-level multimodality for NATs (the lower, the better). Given a translation dataset $\mathcal{D}$ we have the NCM as: $\frac{\mathbb{E}_{(\boldsymbol{x},\boldsymbol{y})\sim\mathcal{D}}[-\log p(\boldsymbol{y}|\boldsymbol{x})]}{\mathbb{E}_{\boldsymbol{y}\sim\mathcal{D}}[\boldsymbol{|\boldsymbol{y}|}]}$. A non-autoregressive model concurrently considering numerous potential translations, it tends to distribute excess probability across different sequences. As a result, it ends up creating translations that are less likely. We validate the NCM of \method and DAT on WMT14 En$\rightarrow$De and De$\rightarrow$En translation task. Results show that we reduce the NCM of DAT from 0.86 to 0.72 on En$\rightarrow$De translation and from 0.91 to 0.79 on De$\rightarrow$En.

But there is no widely accepted metric for modality learning of NATs, we also provide a case study from WMT14 De$\rightarrow$En in Table~\ref{tab:wmt-example} to showcase that the first-order dependency DP loss can not guarantee generated texts are single modal. The top-3 hypotheses sampled from DAT in Table~\ref{tab:wmt-example} are all mixed modal translations. For example, the first hypothesis from DAT is a mixed translation of "\texttt{There is a 50\% increase in parking charges}"  and "\texttt{parking charges should be increased by 50\%}" due to ``charges'' and ``should'' are very likely to appear together in the training corpus so that they will be assigned a high transition score. But mixing words and phrases from the two modalities will always result in a bad translation while evaluating. We show that the problem can be alleviated via training with contrastive constraints which improves the ranking of high-quality and single-modal hypotheses (there is only 1 translation mixing words from different modalities results).

\paragraph{Scaling with Decoder Length and Sampling Size} We further examine how the decoder length of DAT affects our method. Notice that a longer decoder length often provides an enlarged sampling space. We vary the decoder length from 128 to 1,024 and measure the performance gap between \method and  DAT according to BLEU on the test set of WMT14 En$\rightarrow$De translation task. Results are shown in Figure~\ref{fig:lambda}, when we have a decoder length longer than 384, \method generally has an improvement of more than 0.5 BLEU.  \method (decoder length = 384) realizes on-par performance with DAT (decoder length = 1024).
However, when the decoder length is set to a smaller value of 256, the improvement seems marginal. Recalling that \method advocates modality learning by optimizing  $D(q \parallel p)$ with samples from model distribution $\boldsymbol{y}'\sim q(\cdot|\boldsymbol{x})$. We attribute the low performance of \method  to the reduced sampling space when with a short decoder length. We measure the sampling space with a heuristic method: BLEU of the best sample. We respectively calculate the best result among the decoded 64 samples (named Oracle BLEU) with decoder lengths set to 256 and 512. The longer length achieves 34.12 Oracle BLEU while the shorter length only achieves 30.28 Oracle BLEU. The performance gap indicates the size of the sampling space plays an important role in our method. We further examine the performance of \method with sampling size. With the increase in sampling size, single-modal translations also have a higher probability of appearing in the sampling space. According to our training objective, their ranking in the sampled outputs will be elevated, thus improving the model's performance.

\begin{figure}
\centering
\begin{subfigure}{.415\textwidth}
    \begin{center}
        \includegraphics[width=\textwidth]{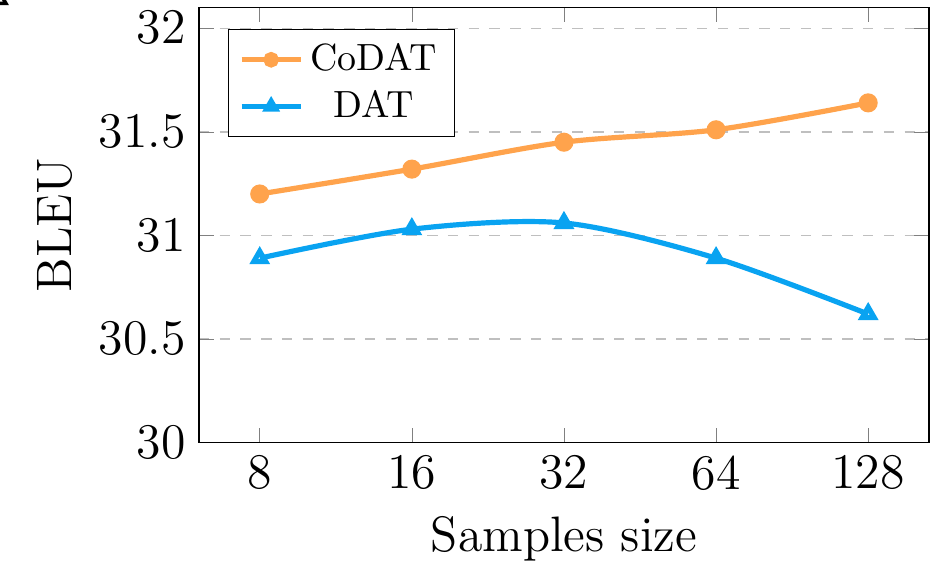}
\caption{Test BLEU score with the sample size on WMT14 De$\rightarrow$En translation.  We use the size of generated samples as the X-axis and the Y-axis represents the BLEU score.}
\label{fig:sample_size}
    \end{center}
\end{subfigure}
\hspace{4mm}
\begin{subfigure}{.42\textwidth}
  \begin{center}
              \vspace{0.1mm}
    \includegraphics[width=\textwidth]{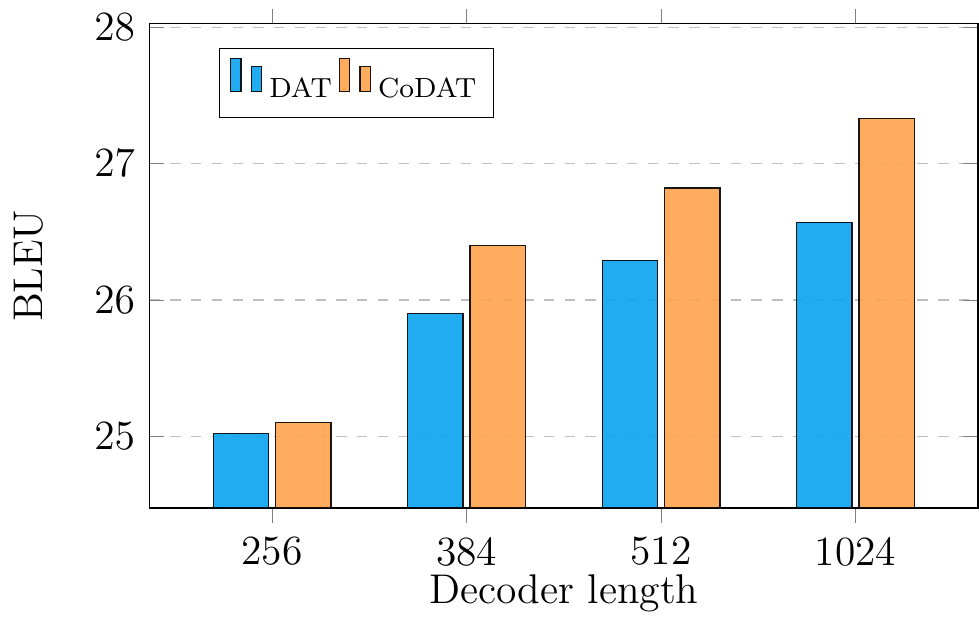}
        \caption{ Test BLEU score of \method and DAT with the decoder length on WMT14 En$\rightarrow$De translation. We use the decoder length as the X-axis and the Y-axis represents the BLEU score. }
    \label{fig:lambda}
  \end{center}
\end{subfigure}
\end{figure}

%% file: 05_relatedwork.tex
\section{Related work}
\paragraph{Non-autoregressive Text Generation}
Non-autoregressive text generation is originally developed in machine translations~\cite{nat2018gu, jmnat2020guo}.
\citet{nat2018gu} design the first NAT model that decodes the whole sequence in parallel. Fully NATs predict the sequence with only one forward pass.  Compared with fully NATs, iterative NATs~\cite{iterativerefinement2018lee, levenshtein2019gu, cmlm2019ghazvininejad, jmnat2020guo}  repeatedly refine the prediction by iterative decoding. Iterative NATs usually achieve better performance but trade efficiency for accuracy. As introduced in the background, non-iterative NATs are greatly improved by DP training~\cite{ctc2018libovicky, axe2020ghazvininejad, huang2022directed} which abandons the position-wise alignment of Vanilla NAT and allows flexible length. Based on DAT, the latest work FDAT~\cite{ma2023fuzzy} introduces fuzzy alignment for DAT and greatly improves the performance on translation tasks but sacrifices the generation diversity.   Non-autoregressive models are also introduced in other areas. \citet{wiseman2018learning} propose a hidden semi-Markov model for data-to-text generation. BANG~\cite{qi2021bang} and MIST~\cite{jiang2021improving} are  pretraining frameworks for NAT models. \citet{su2021non, sun2019fast} successfully utilize pre-trained encoder BERT and a conditional random field (CRF) to non-autoregressive summarization and machine translation. 
Most NATs strongly depend on sentence-level knowledge distillation~\cite{gu2018nonautoregressive} datasets. This requires pretraining an autoregressive model to generate the training set but is very useful to reduce the data modality~\cite{kdnat2020zhou}.

\paragraph{Contrastive Learning for Text Generation}
The triplet loss used in this work is originally proposed in faceNet~\cite{schroff2015facenet}, aiming to learn a good representation via contrasting positives to negatives.  
Contrastive learning has also been widely used in autoregressive text generation~\cite{lee2020contrastive, an2022cont, li2022keywords} and summarization~\cite{liu2021simcls, su2022contrastive}, and to the best of our knowledge, we are the first contrastive learning framework in non-autoregressive generation. In the field of text summarization, \citet{zhong2020extractive} propose contrastive learning to model that the document and its summary should be close in the embedding space.  Previous work on autoregressive generation~\cite{lee2020contrastive, liu2022brio, an2022cont} shares a common motivation: mitigating the exposure bias problem. The contrastive loss we used in training has a similar format to the training objective of previous autoregressive models BRIO~\cite{liu2022brio} and CoNT~\cite{an2022cont}. Our work differs from theirs in two aspects. 1) Different settings: our method is based on the setting of DP-based non-autoregressive models which has a very different input space and training process compared with the autoregressive model. 2) Different motivations: we hold the motivation of helping the modality learning of NATs and our contrastive objective is derived from seeking contrastive constraints while they hold the motivation from exposure bias and learning fine-gained representations.

%% file: 06_appendix.tex
\section{Reverse KL and Reinforcement Learning}
\label{app:rkl-rl}
Above all, we treat $p(\boldsymbol{y}|\boldsymbol{x})$ as Eq.~\ref{eq:pr}.
\begin{align*}
    &\;\;\;\;\text{KL}(q\parallel p) \\
    &= \mathbb{E}_{\boldsymbol{y}'\sim q(\cdot|\boldsymbol{x})}[\log \frac{q(\boldsymbol{y}'|\boldsymbol{x})}{p(\boldsymbol{y}'|\boldsymbol{x})}] \\
    & = \mathbb{E}_{\boldsymbol{y}'\sim q(\cdot|\boldsymbol{x})}[\log q(\boldsymbol{y}'|\boldsymbol{x})-R(\boldsymbol{y}, \boldsymbol{y}')+\log Z^R] \\
    & = -H(q)-\mathbb{E}_{\boldsymbol{y}'\sim q(\cdot|\boldsymbol{x})}[R(\boldsymbol{y}, \boldsymbol{y}')-\log Z^R]
\end{align*}
where $H(q)$ is the entropy of model distribution, $\mathbb{E}_{\boldsymbol{y}\sim q(\cdot|\boldsymbol{x})}[R(\boldsymbol{y}, \boldsymbol{y}')-\log Z^R]$ is reinforcement learning with a normalized reward.
Thus minimizing $\text{KL}(q\parallel p)$ implicitly maximizes rewards on model distribution. 
\section{Sufficiency \& Necessity}
\label{app:proof-sufficient-necessary}

\begin{proposition}
For any sequence pair $\boldsymbol{y}’_1,\boldsymbol{y}'_2$, $\boldsymbol{y}’_1,\boldsymbol{y}'_2 \sim p^R(\boldsymbol{y}'|\mathbf{x},\boldsymbol{y})=\frac{1}{Z^R}\exp(R(\boldsymbol{y},\boldsymbol{y}'))$ is sufficient and necessary to $\log \frac{p^R(\boldsymbol{y}'_1|\mathbf{x},\boldsymbol{y})}{p^R(\boldsymbol{y}'_2|\mathbf{x},\boldsymbol{y})}=R(\boldsymbol{y}'_1,\boldsymbol{y})-R(\boldsymbol{y}'_2,\boldsymbol{y})$.
\end{proposition}

\begin{proof}
We prove the sufficiency and necessity individually as follows:

(Sufficiency)
$$
\forall \boldsymbol{y}’_1,\boldsymbol{y}'_2:
\log \frac{p^R(\boldsymbol{y}'_1|\mathbf{x},\boldsymbol{y})}{p^R(\boldsymbol{y}'_2|\mathbf{x},\boldsymbol{y})}=R(\boldsymbol{y}'_1,\boldsymbol{y})-R(\boldsymbol{y}'_2,\boldsymbol{y})
$$
$$
\forall \boldsymbol{y}’_1,\boldsymbol{y}'_2:
\log p^R(\boldsymbol{y}'_1|\mathbf{x},\boldsymbol{y}) - R(\boldsymbol{y}'_1) \quad\quad\quad\quad\quad\quad
$$
$$
= \log p^R(\boldsymbol{y}'_2|\mathbf{x},\boldsymbol{y}) - R(\boldsymbol{y}'_2)
$$
Without loss of generality, we set the difference at both side to a constant $C$
$$
\forall \boldsymbol{y}': \log p^R(\boldsymbol{y}'|\mathbf{x},\boldsymbol{y}) - R(\boldsymbol{y}') = C
$$
Thus
$$
\forall \boldsymbol{y}': p^R(\boldsymbol{y}'|\mathbf{x},\boldsymbol{y})=\exp(C)\exp(R(\boldsymbol{y},\boldsymbol{y}'))
$$
By rewriting this into a normalization form, we have 
$$
\forall \boldsymbol{y}': p^R(\boldsymbol{y}'|\mathbf{x},\boldsymbol{y})=\frac{1}{Z^R}\exp(R(\boldsymbol{y},\boldsymbol{y}'))
$$
where $Z^R=1/\exp(C)$.

(Necessity)
The necessity is quite straightforward by directly replacing $p^R(\boldsymbol{y}'_1|\mathbf{x},\boldsymbol{y})=\frac{1}{Z^R}\exp(R(\boldsymbol{y},\boldsymbol{y}'_2))$ and $p^R(\boldsymbol{y}'_2|\mathbf{x},\boldsymbol{y})=\frac{1}{Z^R}\exp(R(\boldsymbol{y},\boldsymbol{y}'_2))$ into $\log \frac{p^R(\boldsymbol{y}'_1|\mathbf{x},\boldsymbol{y})}{p^R(\boldsymbol{y}'_2|\mathbf{x},\boldsymbol{y})}$.

\end{proof}

\section{Contrastive Discrepancy}
\label{app:contrastive}

The original Eq.~\ref{eq:reverse-kl} can be rewritten to:
\begin{align*}
    &\mathrm{KL}(q \parallel p)=\mathbb{E}_{\boldsymbol{x},\boldsymbol{y} \sim \mathcal{D}, \boldsymbol{y}'\sim q(\cdot|\boldsymbol{x})}[M_{p,q}(\boldsymbol{y}'|\boldsymbol{x})] \\
    &=\mathbb{E}_{\boldsymbol{x},\boldsymbol{y} \sim \mathcal{D}, \{\boldsymbol{y}'_k\}_{k=1}^K \sim  q(\cdot|\boldsymbol{x})}[\frac{1}{K}\sum_{k}M_{p,q}(\boldsymbol{y}'_k|\boldsymbol{x})].
\end{align*}

We can also understand the contrastive loss from the perspective of the relative version of the reverse KL loss in original Eq.~\ref{eq:reverse-kl}. let  $[\boldsymbol{y}'_i, \boldsymbol{y}'_j]$  denote the difference between a sequence pair from the model distribution, and $\boldsymbol{y}$ is from the data distribution. If there is a significant disparity in the data distribution of the two samples, we will also expect a commensurate divergence in their respective model distribution.
$M_{p,q}([\boldsymbol{y}'_i, \boldsymbol{y}'_j]|\boldsymbol{x}))]$ measures the discrepancy between model and data distribution given the difference of the pair instead of directly optimizing with only one sample like the original Eq.~\ref{eq:reverse-kl}. We can estimate $p([\boldsymbol{y}'_i, \boldsymbol{y}'_j]|\boldsymbol{x})$ with a reward distribution: $\frac{1}{Z_1}\cdot \frac{e^{R(\boldsymbol{y}'_i, \boldsymbol{y})}} { e^{R(\boldsymbol{y}'_j, \boldsymbol{y})}} $ and  $q([\boldsymbol{y}'_i, \boldsymbol{y}'_j]|\boldsymbol{x}) $ with $ \frac{1}{Z_2}\cdot \frac{q(\boldsymbol{y}'_i|\boldsymbol{x})}{q(\boldsymbol{y}'_j|\boldsymbol{x})}$ where $Z_1$ and $Z_2$ are normalizing constants. By optimizing $M_{p,q}([\boldsymbol{y}'_i, \boldsymbol{y}'_j]|\boldsymbol{x}))]$, the Eq.~\ref{eq:final-constraints} also has a similar form with Eq.~\ref{eq:reverse-kl}: 
\begin{align*}
    &\mathbb{E}_{\boldsymbol{x},\boldsymbol{y} \sim \mathcal{D}, \{\boldsymbol{y}'_k\}_{k=1}^K \sim  q(\cdot|\boldsymbol{x})}[\frac{2}{K(K-1)}\sum_{i>j}M_{p,q}([\boldsymbol{y}'_i, \boldsymbol{y}'_j]|\boldsymbol{x})] \\
    =&\mathbb{E}_{\boldsymbol{x},\boldsymbol{y} \sim \mathcal{D}, \{\boldsymbol{y}'_k\}_{k=1}^K \sim  q(\cdot|\boldsymbol{x})}[\frac{2}{K(K-1)}\sum_{i>j}\log \frac{p([\boldsymbol{y}'_i, \boldsymbol{y}'_j]|\boldsymbol{x})}{q([\boldsymbol{y}'_i, \boldsymbol{y}'_j]|\boldsymbol{x})}] \\
    \approx&\mathbb{E}_{\boldsymbol{x},\boldsymbol{y} \sim \mathcal{D}, \{\boldsymbol{y}'_k\}_{k=1}^K \sim  q(\cdot|\boldsymbol{x})}[\frac{2}{K(K-1)}\sum_{i>j}\max \{0, \\
    & \;\;\;\;\;\;\;\; -\log q(\boldsymbol{y}'_i|\boldsymbol{x})+\log q(\boldsymbol{y}'_j|\boldsymbol{x})+(i-j)\epsilon_\text{LB}\}] \\
\end{align*}

\section{Experimental Setup}\label{sec:setup}
\paragraph{Implementation} 
In order to sample meaningful positive and negative examples, we first warmup DA-Transformer with the DP loss shown in Eq.~\ref{eq:dp-loss}. In the warmup stage, we follow the default settings of the open source implementation\footnote{\url{https://github.com/thu-coai/DA-Transformer}} of DA-Transformer.  The learning rate for the machine translation task is set to $5\times10^{-4}$ but for summarization and paraphrasing, we use a smaller learning rate $1\times10^{-4}$. Excluding contrastive learning, we also apply DP training and glancing to help reduce the multi-modality of data distribution. For glancing, we linearly decrease the unmasking ratio from 0.5 to 0.1 until the training steps = 200k.
For DP training, we extend the decoder length to 8 times of the encoder length. We set the margin value $\epsilon_\text{LB}$ to 0.001. The sampling size in both  training and inference is 128 and we only keep the 25\% for training. 
We set the sampling temperature $\tau$ to 0.1 during training and $0.05$ during inference. 
The contrastive loss and DP loss are optimized with the same Adam optimizer~\cite{kingma2014adam}.   We use a batch size of 256k tokens by reducing the frequency of gradient updates to stabilize the training process. After pretraining the model for 120 epochs with DP loss, we further train the model for 5 epochs with the sum of DP loss and contrastive loss.  Our code base is built on fairseq~\cite{ott-etal-2019-fairseq} toolkit. It takes about 1 hour on 4 NVIDIA Tesla A100 GPUs for contrastive optimization on the XSum summarization dataset and 4 hours for the large-scale WMT14 En-De translation task. We assess  BLEU  on the validation set every 40 steps, and then the top 5 checkpoints are saved and averaged to produce the result on the test set. We use the base version of Transformer~\cite{transformer2017vaswani} with 6 encoders and 6 decoders as the autoregressive baseline.  The decoding speedup is tested with the batch size set to 1 following previous work on single A100 GPU and we report the average speedup of three runs. We calculate ROUGE with the pyrouge library\footnote{\url{https://github.com/bheinzerling/pyrouge}}. Both the tokenized BLEU and sacreBLEU are calculated with fairseq.
During decoding, We use the sampling algorithm to derive multiple hypotheses from the model by Algorithm~\ref{alg:inference}. We decompose the sequence probability into transition score ($ q(\mathbf{a} | \boldsymbol{x})$) and prediction score ($ q(\boldsymbol{y} | \boldsymbol{x}, \mathbf{a})$). the sequence score is given by: $\beta \cdot q(\boldsymbol{y} | \boldsymbol{x}, \mathbf{a}) + (1-\beta) \cdot q(\mathbf{a} | \boldsymbol{x}) $ where beta is a hyperparameter tuned on the validation set. We tune $\beta$ from [0.3, 0.5, 0.7]. The validation loss for the checkpoint we used for WMT14 En-De translation is 2.134, so we set $\alpha$ to 2.1 which means training samples with DP Loss larger than 2.1 will not be optimized with samples from the model distribution for stability.


\begin{figure}[ht!]
\centering
    \includegraphics[width=1.0\linewidth]{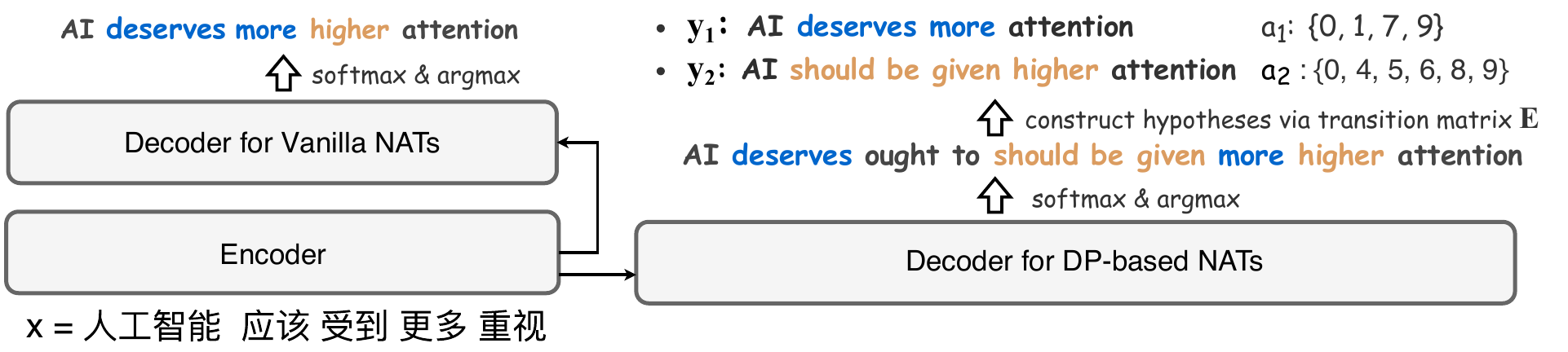}
\vspace{-0.5em}
\caption{Comparasion of Vanilla NATs and DP-based NATs. The input/output length of Vanilla NATs is usually predicted by the model, and the prediction from each position is directly taken as model output. While the decoder length of DP-based models such as CTC~\cite{ctc2018libovicky} and DAT~\cite{huang2022directed} is a hyperparameter (usually set as 2x$\sim$8x encoder length). For DAT, the decoder first predicts a word from each position, and the hypothesis is a subset of all predicted words constructed via the transition matrix. Given a hypothesis $\boldsymbol{y}_1$ from positions $\mathbf{a_1}=\{0,1,7,9\}$, $q(\mathbf{a_1} | \boldsymbol{x})$ is calculated as $\mathbf{E}_{0,1}\cdot\mathbf{E}_{1,7}\cdot\mathbf{E}_{7,9}$ and $q(\boldsymbol{y_1}|\mathbf{a_1}, \boldsymbol{x})$ is the product of the corresponding token probability $p(\texttt{`AI'}|0, \boldsymbol{x})\cdot p(\texttt{`deserves'}|1, \boldsymbol{x})\cdot p(\texttt{`more'}|7, \boldsymbol{x})\cdot p(\texttt{`attention'}|9, \boldsymbol{x})$. }
\label{fig:dp_vanilla}
\end{figure}

\section{Details of DAT Inference}\label{sec:codat}
Latest work CTC~\cite{ctc2018libovicky} and DA-Transformer~\cite{huang2022directed}  adopt a Dynamic Programming (DP) training framework significantly improving  WMT14 en-de translation benchmark by $4\sim5$ BLEU.

In this part, we present a concrete demonstration to perform sampling on DA-Transformer. 
Given the decoder length $L$ and decoder hidden states $ \mathbf{H} = \{\mathbf{h}_0,\mathbf{h}_1, \ldots, \mathbf{h}_L  \}$, we can get the token prediction distribution by: $\mathrm{softmax}(\mathbf{W}_{vocab}\mathbf{H})$ and the $L$ predicted tokens by argmax function. 

Figure~\ref{fig:dp_vanilla} is a specific case when $L=10$. Notably, the $L$ words can not be viewed as the final output in DAT. Compared with vanilla NAT, a hypothesis in DA-Transformer is represented as a small subset of the $L$ predicted tokens which means sampling  hypotheses is equivalent to combining different tokens from the model's prediction. Notably, this can be done efficiently without any additional forward passes. Sampling a subset from these predictions  usually depends on the transition matrix $\mathbf{E} \in \mathbb{R}^{L\times L}$ trained with the DP loss in Eq.~\ref{eq:dp-loss} where $\mathbf{E}_{ij}$ means the transition probability between position $i$ and $j$. Every hypothesis from DAT starts from position 0 and the next token depends on $\arg\max (\mathbf{E}_{[0,:]})$, and the generation process stops when decoding the `<eos>' token. This is the simplest version of decoding in DAT. The commonly used decoding algorithm in DAT is called \textsc{lookahead}. Details of the algorithm are shown in Algorithm~\ref{alg:lookahead}. By replacing the $\arg\max$ function with  multinomial sampling, we can get diverse output parallelly (see Algorithm~\ref{alg:inference}).

\begin{algorithm}[htbp]
\begin{footnotesize}
\caption{\footnotesize  Given decoder length $L$, transition matrix 
$\mathbf{E} \in \mathbb{R}^{L\times L}$ and the token probability vector $\mathbf{t} \in \mathbb{R}^{L}$; return decoder indexes of the output }
\label{alg:lookahead}
\begin{algorithmic}[1]
\Procedure{Lookahead}{}
    \State $\mathbf{P}$ = $\mathbf{E}$ + $\mathbf{t}$.\textsc{unsqueeze}({dim}=0) \algorithmiccomment{consider both transition probability and prediction probability }
    \State $i := 1$, \textit{output} := \textsc{Zeros} [max\_step] 
    \Repeat
        \State \textit{dist} := $\mathbf{P}$ [\textit{outputs}[$i-1$]] \algorithmiccomment{get the distribution of current step given the previous step output}
        \State \textit{output}[$i$] := \textsc{argmax}(\textit{dist})
    \Until $i$ = \textit{max\_step};
    \State \Return \textit{output}
\EndProcedure
\end{algorithmic}
\end{footnotesize}
\end{algorithm}

\begin{algorithm}[htbp]
\begin{footnotesize}
\caption{\footnotesize  Given decoder length $L$, temperature $\tau$, transition matrix 
$\mathbf{E} \in \mathbb{R}^{L\times L}$ and the token probability vector $\mathbf{t} \in \mathbb{R}^{L}$; return decoder indexes of the output }
\label{alg:inference}
\begin{algorithmic}[1]
\Procedure{SAMPLING}{}
    \State $\mathbf{t}$ = $\mathbf{t}$ / \textsc{sum}($\mathbf{t}$)  
    \State $\mathbf{P}$ = $\mathbf{E}$ + $\mathbf{t}$.\textsc{unsqueeze}({dim}=0)  
    \State $\mathbf{P}$ = \textsc{Topp\_Filter}( $\mathbf{P}$, $p=0.5$) 
    \algorithmiccomment{ skip positions with low probability }
    \State $\mathbf{P}$ = \textsc{SoftMax}( $\mathbf{P}$ / $\tau$, {dim}=1)
    \algorithmiccomment{ re-normalize after filtering}
    \State $i := 1$, \textit{output} := \textsc{Zeros} [\textit{max\_step}] 
    \Repeat
        \State \textit{dist} := $\mathbf{P}$ [\textit{outputs}[$i-1$]] \algorithmiccomment{get the distribution of current step given the previous step output}
        \State \textit{output}[$i$] := \textsc{Multinomial}(\textit{dist}, num=1)
    \Until $i$ = \textit{max\_step};
    \State \Return \textit{output}
\EndProcedure
\end{algorithmic}
\end{footnotesize}
\end{algorithm}

\section{Limitations}\label{limit}
We summarize our limitations from three aspects. Firstly, the performance increase brought by contrastive learning-based training is closely related to the ability of the backbone model. The base model should  generate a set of diverse hypotheses and contains high-quality and single-modal samples which can be measured with Oracle BLEU. This limits our contrastive training procedure from directly starting from a random initialized model distribution, and it needs the parameters of the model distribution to have been previously trained to a good parameter space that can provide high-quality samples. Considering the training samples of contrastive loss are from the model distribution instead of data distribution the improvements over strong baselines (especially for the pretrained  model) are usually more remarkable.

Secondly, \method also has impacts on the training and inference efficiency. \method needs further training for 5 epochs.  Our approach also has a minor impact on inference speed.  Since the loss helps the hypothesis with higher bleu have a higher ranking, using sampling algorithms described in Algorithm~\ref{alg:inference} during inference instead of only decoding one hypothesis usually boosts the performance. During inference, we use a sampling size of 128 and we implement the sampling algorithm with the native torch operations for simplicity. 

In practice, using the sampling described in Algorithm~\ref{alg:inference} is not feasible. Calling torch.MULTINOMIAL every step which is significantly slower than torch.ARGMAX resulting in only a 7.2x speedup ratio compared with the AT counterpart on WMT14 En-De translation task. To omit to call torch.MULTINOMIAL in the loop, we change the algorithm by  sampling \textit{max\_step} tokens at each position in parallel and directly obtain the output token instead of sampling every step.  Please refer to our updated sampling algorithm Algorithm~\ref{alg:inference_updated}.
Despite we have optimized the sampling algorithm numerous times, solely with native torch operations still cannot achieve the same decoding speed as the original DAT that decodes only a single hypothesis.

\begin{algorithm}[htbp]
\begin{footnotesize}
\caption{\footnotesize  Given decoder length $L$, temperature $\tau$, transition matrix 
$\mathbf{E} \in \mathbb{R}^{L\times L}$ and the token probability vector $\mathbf{t} \in \mathbb{R}^{L}$; return decoder indexes of the output }
\label{alg:inference_updated}
\begin{algorithmic}[1]
\Procedure{SAMPLING}{}
    \State $\mathbf{t}$ = $\mathbf{t}$ / \textsc{sum}($\mathbf{t}$)  
    \State $\mathbf{P}$ = $\mathbf{E}$ + $\mathbf{t}$.\textsc{unsqueeze}({dim}=0)  
    \State $\mathbf{P}$ = \textsc{Topp\_Filter}( $\mathbf{P}$, $p=0.5$) 
    \algorithmiccomment{ skip positions with low probability }
    \State $\mathbf{P}$ = \textsc{SoftMax}( $\mathbf{P}$ / $\tau$, {dim}=1)
    \algorithmiccomment{ re-normalize after filtering}
    \State $i := 1$, \textit{output} := \textsc{Zeros} [\textit{max\_step}] 
    \State \textit{sampled\_tokens} = \textsc{Multinomial}(P ,\textit{max\_step})
    \algorithmiccomment{ shape $L\times $ \textit{max\_step}}, sampling \textit{max\_step} tokens at each decoder position

    \Repeat
        \State \textit{output}[$i$] := \textit{sampled\_tokens}[\textit{output}[$i-1$]][$i$]
        
    \Until $i$ = \textit{max\_step};
    \State \Return \textit{output}
\EndProcedure
\end{algorithmic}
\end{footnotesize}
\end{algorithm}